\def\year{2019}\relax
\documentclass[letterpaper]{article}

\usepackage{aaai19}
\usepackage{times}
\usepackage{helvet}
\usepackage{courier}
\usepackage{url}
\usepackage{graphicx}
\usepackage{multirow,array,wrapfig,lipsum,booktabs}
\usepackage{amsfonts}       %
\usepackage{microtype}      %
\usepackage{amsmath, amssymb, amsthm}
\usepackage{bbm}
\usepackage{tikz}
\usepackage{pgfplots}
\pgfplotsset{compat=1.14}
\usepackage{nameref}
\usepackage{algorithm} %
\usepackage{algorithmic} %
\DeclareMathOperator{\tr}{tr}

\DeclareMathOperator*{\cov}{cov}
\DeclareMathOperator{\spn}{span}
\newtheorem{theorem}{Theorem}

\newtheorem{lemma}{Lemma}
\newcommand{\T}{\mathrm{T}}
\newcommand{\F}{\mathrm{F}}

\newcommand{\ie}{\textit{i.e. }}

\title{An Efficient Approach to Informative Feature Extraction \\ from Multimodal Data}
\author{Lichen Wang$^{{1}}$\thanks{This work was done when Lichen Wang was an intern at Tencent AI Lab.}, Jiaxiang Wu$^{{2}}$, Shao-Lun Huang$^{{1}}$, Lizhong Zheng$^{{3}}$,\\
{\bf \Large Xiangxiang Xu$^{{4}}$, Lin Zhang$^{{1}}$, Junzhou Huang$^{{5}}$}\\
$^{1}$ Tsinghua-Berkeley Shenzhen Institute, Tsinghua University, $^{2}$ {Tencent AI Lab}\\
$^{3}$ Department of EECS, Massachusetts Institute of Technology\\
$^{4}$ Department of Electronic Engineering, Tsinghua University\\
$^{5}$ Department of CSE, The University of Texas at Arlington\\
Email: wlc16@mails.tsinghua.edu.cn, jonathanwu@tencent.com, shaolun.huang@sz.tsinghua.edu.cn, \\ lizhong@mit.edu, xuxx14@mails.tsinghua.edu.cn, linzhang@tsinghua.edu.cn,
jzhuang@uta.edu}

\begin{document}
\maketitle
\begin{abstract}
One primary focus in multimodal feature extraction is to find the representations of individual modalities that are maximally correlated. As a well-known measure of dependence, the Hirschfeld-Gebelein-R\'{e}nyi (HGR) maximal correlation becomes an appealing objective because of its operational meaning and desirable properties. However, the strict whitening constraints formalized in the HGR maximal correlation limit its application. To address this problem, this paper proposes Soft-HGR, a novel framework to extract informative features from multiple data modalities. Specifically, our framework prevents the ``hard'' whitening constraints, while simultaneously preserving the same feature geometry as in the HGR maximal correlation. The objective of Soft-HGR is straightforward, only involving two inner products, which guarantees the efficiency and stability in optimization. We further generalize the framework to handle more than two modalities and missing modalities. When labels are partially available, we enhance the discriminative power of the feature representations by making a semi-supervised adaptation. Empirical evaluation implies that our approach learns more informative feature mappings and is more efficient to optimize.
\end{abstract}

\noindent %

\section{Introduction}
Human perception is typically more accurate when objects are presented in multiple modalities, as information from one sense often augments information from another. The idea has risen recent interests to develop learning machines which can extract correlation across modalities, through the perception of equivalence, dependence or association. However, compared to the ease of human perception, identifying the relationship among multiple sources is much harder for machines. The reason lies in the facts that the varying statistic properties carried by data from each source obscure the correlation among modalities, which could be vital for learning effective feature representations \cite{baltrusaitis2018multimodal,sohn2014improved}. Existing methods approaches this problem by Canonical Correlation Analysis (CCA) \cite{hotelling1936relations,akaho2006kernel,andrew2013deep}, Euclidean distance minimization \cite{frome2013devise}, enforcing partial order \cite{vendrov2015order}, etc.

In statistic, the Hirschfeld-Gebelein-R\'{e}nyi (HGR) maximal correlation \cite{hirschfeld1935connection,gebelein1941statistische,renyi1959measures}, as a generalization from the Pearson's correlation \cite{pearson1895note}, is well-known for its legitimacy as a measure of dependence. Such notion is appealing to multimodal feature extraction for many reasons. For example, maximizing the HGR maximal correlation enables us to determine the nonlinear transformations of two variables that are maximally correlated \cite{feizi2017network}. In the perspective of the information theory, the HGR transformation carries the maximum amount of information of $X$ about $Y$, and vice versa \cite{huang2017information1}. As for generality, CCA \cite{hotelling1936relations} and its variants \cite{bach2002kernel,akaho2006kernel,andrew2013deep} can be regarded as the realizations of the HGR maximal correlation with different designs of transformation functions.

However, the HGR maximal correlation suffers from two limitations. Firstly, HGR maximal correlation involves whitening constraints which require each feature to be strictly uncorrelated. Most commonly, the orthogonal geometry is preserved by a whitening process \cite{andrew2013deep,wang2015stochastic}, which relies on the computation of matrix inversion or decomposition. These operations are of high-complexity and may have numerical stability issues for large feature dimensions. Secondly, discriminativeness is not explicitly formulated in the objective of the HGR maximal correlation. In fact, it can lead to desirable performance in downstream supervised tasks only if all the discriminative information ``accidentally'' lies in the common subspace of different modalities. Such assumption may not hold true when input modalities are weakly correlated and do not possess much common information. In this case, the underlying discriminative information is more likely to be omitted after feature mapping, which leads to performance degradation. 

To address these problems, we propose Soft-HGR, a novel framework to learn correlated representation across modalities without hard whitening constraints. The objective of Soft-HGR consists of two inner products, one between the feature mappings and the other between feature covariances. While the formulation rules out the whitening constraints, our model is still able to preserve the same feature geometry as in the original HGR formulation. Therefore, no additional decorrelation process is required in optimization, which promises scalability and stability to the algorithm. Besides, the simple formulation of the Soft-HGR provides additional generalizability to the framework. Soft-HGR can be readily extended to manage more than two modalities and missing modalities. In the semi-supervised settings, we adapt the model to extract the information not only about the dependence between different modalities, but has good predictive power to the labels. Empirically, our method reveals superior efficiency, stability and discriminative performance on real data.

In summary, our main contributions are as follows:
\begin{itemize}
    \item We proposed Soft-HGR, based on the HGR maximal correlation, to extract informative features from multimodal data. The objective is simple and easy to implement;
    \item We proposed an alternative strategy to learn the HGR transformations without explicit whitening constraints. The optimization is more efficient and reliable;
    \item We generalize our framework to handle more than two modalities and missing modalities, and to incorporate discriminative information for semi-supervised tasks.
\end{itemize}

\section{Background: The HGR Maximal Correlation}
The HGR maximal correlation \cite{hirschfeld1935connection,gebelein1941statistische,renyi1959measures} generalizes the well-known Pearson's correlation \cite{pearson1895note} as a general measure of dependence. While it was originally defined on one feature, the multi-feature extension is straightforward. For joint distributed random variables $X$ and $Y$ with ranges $\mathcal{X}$ and $\mathcal{Y}$, the HGR maximal correlation with $k$ features is defined by:
\begin{equation} \label{eq:1}
 \rho^{(k)}\left(X, Y\right) = \sup_{\substack{\mathbf{f}:\mathcal{X} \rightarrow \mathbb{R}^{k}, \frac{1}{k}\mathbb{E}\left [\mathbf{f}\right] = 0, \mathrm{Cov}\left(\mathbf{f}\right) = \mathbf{I} \\
\mathbf{g}: \mathcal{Y} \rightarrow \mathbb{R}^{k}, \mathbb{E}\left[\mathbf{g}\right] = 0,\mathrm{Cov}\left(\mathbf{g}\right) = \mathbf{I}}} \mathbb{E}\left[\mathbf{f}^{\T}(X) \mathbf{g}(Y)\right]
\end{equation}
where $\mathbf{f} = [f_1, f_2, ..., f_k]^{\T}$, $\mathbf{g} = [g_1, g_2, ..., g_k]^{\T}$, and the supremum is taken over all sets of Borel measurable functions with zero-mean and identity covariance. As a legitimate measure of dependence, the HGR maximal correlation satisfies many fundamental properties which are rarely provided. For example, the correlation coefficient is bounded by 0 and 1, corresponding to the case when two random variables are independent, or there exists a deterministic relationship
between $X$ and $Y$ \cite{renyi1959measures}.

There are many reasons why HGR maximal correlation is appealing to multimodal feature extraction. For example, finding the HGR maximal correlation also leads us to the non-linear transformation $\mathbf{f}$ and $\mathbf{g}$. These transformations are the most ``informative'' ones, in the view of information theory, as $\mathbf{f}(X)$ carries the maximum amount of the information towards $Y$ and vice versa \cite{huang2017information1}.

\begin{figure}
  \centering
  \includegraphics[scale=0.064]{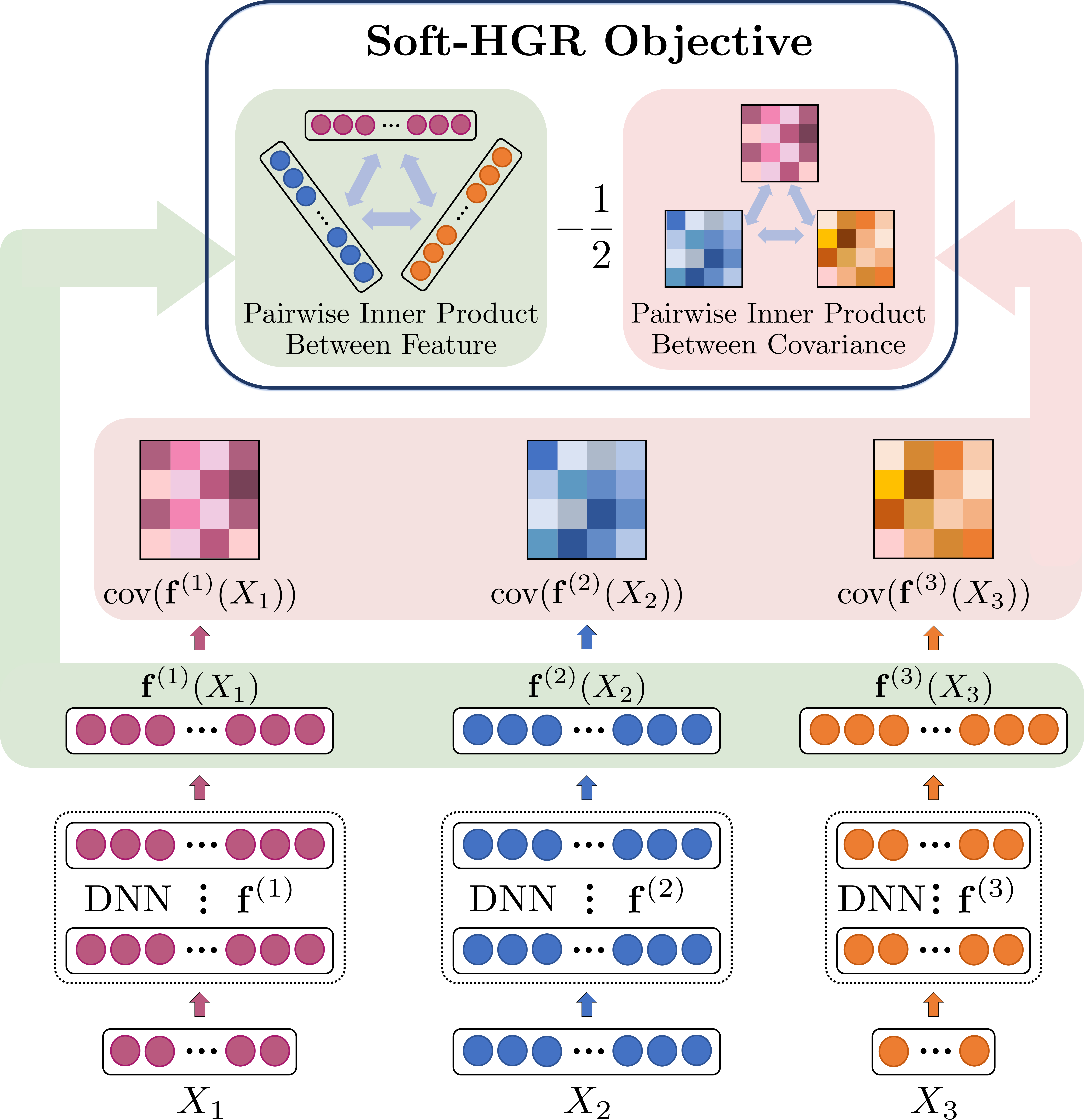}
  \caption{Architecture of Soft-HGR}
  \label{fig:figure_1}
\end{figure}

\subsection{Connections to CCA Based Models}
One strand of research on correlation extraction is based on the work of Hotelling on CCA \cite{hotelling1936relations}, which is later extended to Kernal CCA \cite{bach2002kernel,akaho2006kernel} and Deep CCA \cite{andrew2013deep}. In fact, CCA based models share a very similar objective to the HGR maximal correlation, except their transformation functions are restricted to certain forms. More specifically, CCA and Kernel CCA find optimal feature mappings in linear and reproducing kernel Hilbert space, respectively. Deep CCA takes a different approach, in which the $\mathbf{f}$ and $\mathbf{g}$ are implemented as deep neural networks. Assuming the infinite expressive power of the neural structure, the $\mathbf{f}$ and $\mathbf{g}$ have the capability to approximate the HGR transformations.

\subsection{Limitations}

An impediment to HGR maximal correlation is that the whitening constraints bring high computational complexity to the optimization. Existing models introduce a decorrelation step which forces the covariance to be an identity matrix. The decorrelation process is not scalable since it relies on the computation of the matrices inversion and decomposition, whose time complexity is $O(k^3)$. Besides, the optimization in practice often encounters gradients explosion as we choose large $k$, because the covariance matrices become ill-posed. Some works are proposed to address the problem. Soft-CCA \cite{Chang_2018_CVPR} introduces a decorrelation regularizer based on the $l_1$ penalty to replace the hard whitening constraints. Correlational Neural Network \cite{chandar2016correlational}, inspired by autoencoder, introduces an addition reconstruction loss to replace the whitening constraints. However, both methods break the original feature geometry of the HGR maximal correlation.

Besides, the features extracted from the HGR maximal correlation are not necessarily suitable for downstream discriminative tasks. As a dimension deduction process, there are inevitably some information about data that is discarded during transformation $\mathbf{f}:\mathbb{R}^{\mathcal{X}} \rightarrow \mathbb{R}^{k}$. This is acceptable if the primary goal is to model the correlation between modalities. However, if $\mathbf{f}$ is utilized for future discriminative tasks, we may expect some performance loss. 

\section{Soft-HGR}
In this section, we detail our framework for Soft-HGR. We commence by deriving the optimal solution for the HGR maximal correlation with the whitening constraints. Then we propose an alternative strategy, the low-rank approximation, to approach the HGR problem. we show our proposed objective escapes whitening constraints but still arrives at an equivalent optimum. Finally, we generalize Soft-HGR to handle more than two data modalities and missing modalities, and to incorporate supervised information.

\subsection{The Optimal Feature Transformations}
To simplify the discussions, we assume that $X$ and $Y$ are discrete random variables with range $\mathcal{X} = \{1,2,...,|\mathcal{X}|\}$ and $\mathcal{Y} = \{1,2,...,|\mathcal{Y}|\}$, respectively. However, the discussion is still valid when $X$ and $Y$ are multivariate and continuous in nature. 

We first introduce matrix $\mathbf{B} \in \mathbb{R}^{|\mathcal{X}|\times|\mathcal{Y}|}$ as a function of joint distribution $P_{XY}$ \cite{huang2017information1}. The $(x,y)$-th entry is defined as:
\begin{equation}
    B_{x,y} = \frac{P_{XY}(x,y)}{\sqrt{P_X(x)}\sqrt{P_Y(y)}}
\end{equation}
As a summarization of the data, $\mathbf{B}$ has the following property:
\begin{lemma} The largest singular value of $\mathbf{B}$ is $1$, with the corresponding left and right singular vectors given by:
\begin{equation} \label{singular_vector}
    \begin{aligned}
        \mathbf{u}_0 &= \left[ \sqrt{P_X(1)}, \sqrt{P_X(2)}, \dots, \sqrt{P_X(|\mathcal{X}|)} \right]^{\T} \\
        \mathbf{v}_0 &= \left[ \sqrt{P_Y(1)}, \sqrt{P_Y(2)}, \dots, \sqrt{P_Y(|\mathcal{Y}|)} \right]^{\T}
    \end{aligned}
\end{equation}
\end{lemma}
\begin{proof} For any $\boldsymbol{\psi} = [\sqrt{P_Y(y)} \, g(y), \: y = 1,2,\dots,|\mathcal{Y}| ]^{\T}$ that satisfies $||\boldsymbol{\psi}||_2 = 1$, we have
\begin{equation}
\begin{aligned}
    ||\mathbf{B}\boldsymbol{\psi}||_2^2 = &\sum_x \Biggl(\sum_y \frac{P_{XY}(x,y)}{\sqrt{P_X(x)}\sqrt{P_Y(y)}} \sqrt{P_Y(y)} \,g(y)\Biggr)^2 \\
    = & \sum_x P_X(x)\Biggl(\sum_y \frac{P_{XY}(x,y)}{P_X(x)}\,g(y)\Biggr)^2 \\
    = & \sum_x P_X(x) \mathbb{E}^2\left[g(Y)|X=x\right] \\
    \leq & \sum_x P_X(x) \mathbb{E}\left[g^2(Y)|X=x\right] \\
    = & \mathbb{E}\left[g^2(Y)\right] = ||\boldsymbol{\psi}||^2_2 = 1
\end{aligned}
\end{equation}
Therefore the largest singular value $\sigma_0 = \sup ||\mathbf{B}\boldsymbol{\psi}||_2 \leq 1$. The equality only holds when $g(Y)$ is the constant 1 and $\boldsymbol{\psi} = \mathbf{v}_0$. The derivation is similar for $\mathbf{u}_0$.
\end{proof}

Below, we show that finding the most correlated feature transformations for the maximal HGR correlation is equivalent to solving the SVD for $\tilde{\mathbf{B}} = \mathbf{B}-\mathbf{u}_0\mathbf{v}_0^{\T}$.

\begin{theorem} \label{theorem1}
\cite{huang2017information1} Given the SVD of $\mathbf{B} = \mathbf{U}\boldsymbol{\Sigma} \mathbf{V}^{\T} = \sum_{i=0}^{K}\sigma_i\mathbf{u}_i\mathbf{v}_i^{\T}$, with $1 = \sigma_0 \geq \sigma_1 \geq ... \geq \sigma_K$, then optimal feature transformations for the HGR maximal correlation are given by:
\begin{equation} \label{opt_sol}
    \begin{aligned}
    f_i^*(x) &= U_{x,i}/\sqrt{P_X(x)}, i = 1, ..., k, x \in \mathcal{X} \\[1mm]
    g_i^*(y) &= V_{y,i}/\sqrt{P_Y(y)}, i = 1, ..., k, y \in \mathcal{Y}
    \end{aligned}
\end{equation}
\end{theorem}
\begin{proof}
\begin{equation}
\begin{aligned}
      & \mathbb{E}\left[\mathbf{f}^{\T}(X) \mathbf{g}(Y)\right] \\
    = & \sum_{x \in \mathcal{X}}\sum_{y \in \mathcal{Y}}P_{XY}(x,y)\mathbf{f}^{\T}(x)\mathbf{g}(y) \\
    = & \sum_{x \in \mathcal{X}}\sum_{y \in \mathcal{Y}} \sqrt{P_X(x)}\mathbf{f}^{\T}(x)\frac{P_{X,Y}(x,y)}{\sqrt{P_X(x)}\sqrt{P_Y(y)}}\sqrt{P_Y(y)}\mathbf{g}(y) \\
    = & \tr(\boldsymbol{\Phi}^{\T}\mathbf{B}\boldsymbol{\Psi}) \label{PhiBPsi}
\end{aligned}
\end{equation}
In \eqref{PhiBPsi} we introduce new variables $\boldsymbol{\Phi} \in \mathbb{R}^{|\mathcal{X}|\times k}$ and $\boldsymbol{\Psi} \in \mathbb{R}^{|\mathcal{Y}|\times k}$, which are connected to $\mathbf{f}$ and $\mathbf{g}$ by:
\begin{equation}
    \begin{aligned}
        \boldsymbol{\Phi} &= \left[\sqrt{P_X(1)}\mathbf{f}(1), \dots, \sqrt{P_X(|\mathcal{X}|)}\mathbf{f}(\mathcal{|X|})\right]^{\T} \\
        \boldsymbol{\Psi} &= \left[\sqrt{P_Y(1)}\mathbf{g}(1), \dots, \sqrt{P_Y(|\mathcal{Y}|)}\mathbf{g}(\mathcal{|Y|})\right]^{\T} \label{var_sub}
    \end{aligned}
\end{equation}

Following the variables substitution, the objective of the HGR maximal correlation can be reformulated as follows:
\begin{align}
 \rho_k\left(X, Y\right) & = 
 \max_{\substack{
 \mathbf{f}:\mathcal{X} \rightarrow \mathbb{R}^{k}, 
 \mathbb{E}\left [\mathbf{f}\right] = \mathbf{0}, 
 \mathrm{Cov}\left(\mathbf{f}\right) = \mathbf{I} \\
 \mathbf{g}: \mathcal{Y} \rightarrow \mathbb{R}^{k}, 
 \mathbb{E}\left[\mathbf{g}\right] = \mathbf{0},
 \mathrm{Cov}\left(\mathbf{g}\right) = \mathbf{I}}} 
 \mathbb{E}\left[\mathbf{f}^{\T}(X) \mathbf{g}(Y)\right] \\ 
& = \max_{
    \substack{
    \boldsymbol{\Phi}:
    \boldsymbol{\Phi}^{\T}\mathbf{u}_0 = \mathbf{0},
    \boldsymbol{\Phi}^{\T}\boldsymbol{\Phi} = \mathbf{I} \\
    \boldsymbol{\Psi}:
    \boldsymbol{\Psi}^{\T}\mathbf{v}_0 = \mathbf{0},
    \boldsymbol{\Psi}^{\T}\boldsymbol{\Psi} = \mathbf{I}}
    }
\tr(\boldsymbol{\Phi}^{\T}\mathbf{B}\boldsymbol{\Psi})\\
& = \max_{
    \substack{
    \boldsymbol{\Phi}:
    \boldsymbol{\Phi}^{\T}\boldsymbol{\Phi} = \mathbf{I} \\
    \boldsymbol{\Psi}:
    \boldsymbol{\Psi}^{\T}\boldsymbol{\Psi} = \mathbf{I}}
    }
\tr(\boldsymbol{\Phi}^{\T}\tilde{\mathbf{B}}\boldsymbol{\Psi}) \label{Bhat}
\end{align}
As for the optimization problem in \eqref{Bhat}, the optimal $\boldsymbol{\Phi}^*$ and $\boldsymbol{\Psi}^*$ should align the left and right singular vectors of $\tilde{\mathbf{B}}$ respectively. Substituting $\{\boldsymbol{\Phi}^*, \boldsymbol{\Psi}^*\}$ back to $\{\mathbf{f}, \mathbf{g}\}$ leads us to the solution in \eqref{opt_sol}.
\end{proof}

For the maximization problem in \eqref{Bhat}, the whitening constraints over $\boldsymbol{\Phi}^{\T}\boldsymbol{\Phi}$ and $\boldsymbol{\Psi}^{\T}\boldsymbol{\Psi}$ are inevitable as they assure the selected features to be mutually orthogonal in the functional space. In the next subsection, we show an alternative formulation for this problem.

\subsection{Alternative: The Low-rank Approximation}
Instead of solving the SVD, we approach this problem by discovering the low-rank approximation of $\tilde{\mathbf{B}}$, where all the cross-modal interactions lies in. Recall the variable equivalence in \eqref{var_sub}, we approximate the $\tilde{\mathbf{B}}$ by:

\begin{equation} \label{obj:low_rank}
    \begin{aligned}
        \mathop {\min } \limits_{\mathbf{f},\mathbf{g}} & \quad \frac{1}{2} \|\tilde{\mathbf{B}} -\boldsymbol{\Phi}\boldsymbol{\Psi}^{\T}\|_{{\F}}^{2} \\
    s.t. & \quad \mathbb{E}\left[\mathbf{f}(X)\right] = \mathbb{E}\left[\mathbf{g}(Y)\right] = \mathbf{0}.
    \end{aligned}
\end{equation}
Note that we do not impose constraints on the $\cov(\mathbf{f}(X))$ or $\cov(\mathbf{g}(Y))$. We will soon argue that this formulation leads to the same feature geometry as the one in \eqref{Bhat}. In order to solve this problem, we introduce the following theorem:

\begin{theorem} \label{lemma_low_rank}
\textbf{(Eckart-Young-Mirsky Theorem)} \cite{eckart1936approximation} Suppose $\mathbf{A} = \mathbf{U}\boldsymbol{\Sigma} \mathbf{V}^{\T}$, then $\mathbf{A}_r = \mathbf{U}_r\boldsymbol{\Sigma}_r \mathbf{V}_r^{\T}$ $= \sum_{i=1}^{r}\sigma_i\mathbf{u}_i\mathbf{v}_i^{T}$ is the optimal solution to the following low-rank approximation problem:
\begin{equation}
    \begin{aligned}
    \mathop {\min } \limits_{\mathbf{A}_r} & \quad \|\mathbf{A} - \mathbf{A}_r \|_{\F}^{2} \\
    s.t. & \quad \text{\emph{rank}}(\mathbf{A}_r) \leq r.
    \end{aligned}
\end{equation}
\end{theorem}

Therefore, the optimal $\boldsymbol{\Phi}^*$ and $\boldsymbol{\Psi}^*$ should follow:
\begin{equation}
    \boldsymbol{\Phi}^*{\boldsymbol{\Psi}^*}^{\T} = \sum_{i=1}^{k}\sigma_i\mathbf{u}_i\mathbf{v}_i^{T} = \mathbf{U}_{1:k}\boldsymbol{\Sigma}_{1:k}\mathbf{V}_{1:k}^{\T}
\end{equation}

The $\boldsymbol{\Phi}$ and $\boldsymbol{\Psi}$ is not unique. Given any constant decomposition of $\boldsymbol{\Sigma}_{1:k} = \mathbf{H}_{1} \mathbf{H}_{2}^{\T}$, there is an associated solution $\boldsymbol{\Phi}^* = \mathbf{U}_{1:k} \mathbf{H}_{1}, \boldsymbol{\Psi}^* = \mathbf{V}_{1:k} \mathbf{H}_{2}$. Equivalent expression for $\mathbf{f}$ and $\mathbf{g}$ is:
\begin{equation} \label{shgr_opt_sol}
    \begin{aligned}
    f_i^*(x) = \left[\mathbf{U}_{1:k} \mathbf{H}_{1}\right]_{x,i}/\sqrt{P_X(x)}, i = 1, ..., k, x \in \mathcal{X} \\[1mm]
    g_i^*(y) = \left[\mathbf{V}_{1:k} \mathbf{H}_{2}\right]_{x,i}/\sqrt{P_Y(y)}, i = 1, ..., k, y \in \mathcal{Y}
    \end{aligned}
\end{equation}

Since $\mathbf{H_1}$ are $\mathbf{H_2}$ are invertable, one can conclude that the optimal feature transformation for Soft-HGR \eqref{shgr_opt_sol} and for the HGR maximal correlation \eqref{opt_sol} are linearly transformable from the one to the other. Namely, they span the same feature space, \ie $\,\spn\{f_1, f_2, ..., f_k\} = \spn\{f_1^*, f_2^* , ..., f_k^*\}$ (resp. for $\mathbf{g}$) and therefore describe same amount of information. One way to understand this equivalence is to imagine that the HGR features are feed into a linear dense layer, and output Soft-HGR features with same dimensions.

\subsection{The Soft-HGR Objective}
Thus far, we prove that the low-rank approximation of $\tilde{\mathbf{B}}$ also leads to the optimal feature transformation. Based on this idea, now we develop the operational objective for Soft-HGR. By expanding \eqref{obj:low_rank}, we have:
\begin{align}
    &\quad\,\,\frac{1}{2}  \|\tilde{\mathbf{B}} -\boldsymbol{\Phi}\boldsymbol{\Psi}^{\T}\|_{{\F}}^{2} \label{low_rank}\\
    &= \frac{1}{2} \| \tilde{\mathbf{B}} \|_{\F}^2 - \tr(\boldsymbol{\Phi}^{\T}\tilde{\mathbf{B}}\boldsymbol{\Psi})
    + \frac{1}{2}\tr(\boldsymbol{\Phi}^{\T}\boldsymbol{\Phi}\boldsymbol{\Psi}^{\T}\boldsymbol{\Psi}) \label{expan}
\end{align}
where the norm of $\tilde{\mathbf{B}}$ given the data is a constant. Minimizing the last two terms with respect to $\mathbf{f}$ and $\mathbf{G}$ leads us to the Soft-HGR objective:
\begin{equation} \label{obj:softhgr}
    \begin{aligned}
    \mathop {\max } \limits_{\mathbf{f}, \mathbf{g}} & \quad \mathbb{E} \left[\mathbf{f}^{\T}(X)\mathbf{g}(Y)\right]-\frac{1}{2}\tr\left(\cov(\mathbf{f}(X))\cov(\mathbf{g}(Y))\right) \\ 
    s.t. & \quad \mathbb{E}\left[\mathbf{f}(X)\right] = \mathbb{E}\left[\mathbf{g}(Y)\right] = \mathbf{0}.
    \end{aligned}
\end{equation}
The proposed Soft-HGR consists of two inner products, one between feature mappings and the other between feature covariance. The first term in \eqref{obj:softhgr} is consistent to the objective of the HGR maximal correlation, and the second term is considered as a soft regularizer to replace the whitening constraints.

Follow the practice of Deep CCA, we design transformation functions $\mathbf{f}$ and $\mathbf{g}$ as parametric neural networks. As long as the reachable functional space of the neural structures covers the optimal feature transformation, the Soft-HGR and the HGR maximal correlation will always lead us to the equivalent solution.

\begin{algorithm}[!tb]
\caption{Evaluate Soft-HGR on a mini-batch} \label{alg:1}
\begin{algorithmic}[1] %
\REQUIRE ~~\\ %
Paired data samples of two modalities in a mini-batch of size $m$: $(\mathbf{x}^{(1)}, \mathbf{y}^{(1)}), \cdots, (\mathbf{x}^{(m)}, \mathbf{y}^{(m)})$ \\
Two branches of parameterized neural networks with $k$ output units: $\mathbf{f}$ and $\mathbf{g}$
\ENSURE ~~\\ %
The objective value of Soft-HGR
\STATE Subtract the mean of features: \\
$\quad \mathbf{f}(\mathbf{x}^{(i)}) \leftarrow \mathbf{f}(\mathbf{x}^{(i)}) - \frac{1}{m}\sum_{j=1}^{m}\mathbf{f}(\mathbf{x}^{(j)}), i = 1, \cdots, m$ \\ 
$\quad \mathbf{g}(\mathbf{y}^{(i)}) \leftarrow \mathbf{g}(\mathbf{y}^{(i)}) - \frac{1}{m}\sum_{j=1}^{m}\mathbf{g}(\mathbf{y}^{(j)}), i = 1, \cdots, m$
\STATE Compute the empirical covariance: \\
$\quad \cov(\mathbf{f}) \leftarrow \frac{1}{m-1}\sum_{i=1}^{m}\mathbf{f}(\mathbf{x}^{(i)})\mathbf{f}(\mathbf{x}^{(i)})^{\T}$\\
$\quad \cov(\mathbf{g}) \leftarrow \frac{1}{m-1}\sum_{i=1}^{m}\mathbf{g}(\mathbf{y}^{(i)})\mathbf{g}(\mathbf{y}^{(i)})^{\T}$
\STATE Compute the empirical Soft-HGR objective: \\
$\quad \frac{1}{m-1}\sum_{i=1}^{m}f(\mathbf{x}^{(i)})^{\T}g(\mathbf{y}^{(i)}) - \frac{1}{2}\tr(\cov(\mathbf{f})\cov(\mathbf{g}))$
\end{algorithmic}
\end{algorithm}

\subsection{Optimization} In practise, we do not usually have access to the joint probability distribution $P_{XY}$, but rather paired multimodal samples $(\mathbf{x}^{(1)}, \mathbf{y}^{(1)}), \cdots, (\mathbf{x}^{(m)}, \mathbf{y}^{(m)})$ retrived from this distribution. As common practices, we embrace SGD techniques that operate on mini-batch of data to optimize the Soft-HGR. The prominent concern here is how to the estimate of the sample covariance with only partially seen mini-batches. In fact, we find that simply using the batch covariance as a replacement awards the best performance. This implies the Soft-HGR actually decomposes the empirical $\tilde{\mathbf{B}}$ over every mini-batch. Only in this way the empirical $P_{XY}$ is always consistent with the marginal distribution $P_X$ and $P_Y$, where the covariance is evaluated on. The detailed procedure to calculate the Soft-HGR objective is summarized in Algorithm \ref{alg:1}. %
The overall complexity of Soft-HGR is $O(mk^2)$, which is significantly less compared to $O(mk^2+ k^3)$ for normal HGR implementation, \ie Deep CCA. It is also worth noting that our method does not impose an upper bound on the feature dimension $k$. The optimization is consistently stable for very large $k$.

\subsection{Extension to More or Missing Modalities}
The HGR maximal correlation is originally defined on two random variables. In contrast to reconstruction models \cite{srivastava2012multimodal,zhao2015heterogeneous}, the multi-modal extension for correlation based models is not straightforward. New modalities will bring additional whitening constraints, and the computational complexities scales up. However, in Soft-HGR, the ``soft'' formulation provides more flexibility. Recall that the core idea behind the Soft-HGR is to find an approximation of the $\tilde{\mathbf{B}}$ matrix defined on two modalities. In order to handle more than two modalities, the multimodal Soft-HGR should be able to learn feature transformations which recover all pairwise $\tilde{\mathbf{B}}$ simultaneously. Landing on this idea, let $X_1, \dots, X_d$ be $d$ different modalities, and $\mathbf{f}^{(1)}, \dots, \mathbf{f}^{(d)}$ be their corresponding transformation functions, the multimodal Soft-HGR is defined as:
\begin{equation}\label{equ:mul_obj}
\begin{aligned}
\mathop {\max }\limits_{\mathbf{f}^{(1)},\dots, \mathbf{f}^{(d)}} & \quad {\mathbb{E} \left[\sum_{i \neq j}^{d} {\mathbf{f}^{(i)}}^{\T}(X_i){\mathbf{f}^{(j)}(X_j)}\right]} \\[1mm]
& - \frac{1}{2}\sum_{i \neq j}^{d}\tr\left(\cov(\mathbf{f}^{(i)}(X_i))\cov(\mathbf{f}^{(j)}(X_j))\right)\\[1mm]
s.t. \quad & \quad \mathbf{f}^{(i)}:\mathcal{X}_i \rightarrow \mathbb{R}^{k}; \mathbb{E}\left [\mathbf{f}^{(i)}(X_i) \right] = \mathbf{0}; \\
& \quad i,j = 1, 2, \dots, d.
\end{aligned}
\end{equation}

When $d=3$, Figure \ref{fig:figure_1} provides an illustration for \eqref{equ:mul_obj} with neural network implementations. The DNN structure for each neural branches may vary, depending on the statistical property of the inputs. The overall model extracts the features from every neural branch, and maximize their pairwise Soft-HGR in an additive manner. From an information theoretical perspective, maximizing \eqref{equ:mul_obj} is equivalent to extracting the common information from multiple random variables.

Note that this generalization also provides solutions to deal with data with partially missing modalities. To see this, the first term in \eqref{equ:mul_obj} can be applied only on the presented modalities for each training sample, and the second term is always measurable as it only depends on the marginal distribution of individual modalities.

\subsection{Incorporating Supervised Information} \label{sec:semi}

The primary goal of the above framework is to extract the correlation between modalities. Therefore, any information that is private to the individual modality is eliminated, regardless of its discriminative power. The intuition behind the supervised/semi-supervised adaptation is that feature extraction should be conducted under the guidance of supervised labels, even if they are insufficient.

Assumed that a subset of bi-modal data is associated with discrete labels $Z$ with range $\mathcal{Z} = \{1,2,\dots,|\mathcal{Z}|\}$. In order to receive the supervised information from labels, we feed the joint representation, the concatenation of individual feature mappings, into a softmax classifier. The cross entropy loss is added to the overall objective, with a hyper-parameter $\lambda \in [0, 1]$ to trade off the strength of the unsupervised component:
\begin{equation}
    \begin{split} \label{eq:semi}
    \mathcal{L} = &  (\lambda - 1) \cdot \mathbb{E}\left[\log Q_{Z|XY}\right] - \lambda \mathbb{E} \left[\mathbf{f}^{\T}(X)\mathbf{g}(Y)\right] \\
    & \quad + \frac{ \lambda}{2}\tr\left(\cov(\mathbf{f}(X))\cov(\mathbf{g}(Y))\right)
\end{split}
\end{equation} 
where 
\begin{equation}
    Q_{Z=j|XY} = \frac{\exp\left({\left[\mathbf{f}^{\T}(X), \mathbf{g}^{\T}(Y)\right]\boldsymbol{\theta}_j}\right)}{\sum_{i=1}^{|\mathcal{Z}|}\exp\left({\left[\mathbf{f}^{\T}(X), \mathbf{g}^{\T}(Y)\right]\boldsymbol{\theta}_i}\right)}
\end{equation}

In semi-supervised settings, the supervised softmax loss, the first term in \eqref{eq:semi}, is only effective when labels are presented. The last two terms of \eqref{eq:semi} corresponds to the Soft-HGR loss, which is evaluated independently from labels. The gradients from the label $Z$ are first backpropagated to the individual feature mappings, then affect the feature selection. 

\section{Experiments}
In this section, we evaluate Soft-HGR in the following aspects:

\begin{itemize}
    \item To verify the relationship between the HGR features and Soft-HGR feature is linear;
    \item To compare the efficiency and numerical stability of CCA based models and Soft-HGR;
    \item To demonstrate the power of semi-supervised Soft-HGR on discriminative tasks with limited labels;
    \item To show the performance of Soft-HGR on more than two modalities and missing modalities.
\end{itemize}

\begin{table}[t]
  \caption{The linear correlation between features extracted from the Soft-HGR and the HGR maximal correlation.}
  \label{table_1}
  \begin{center}
  \begin{tabular}{lllll}
    \toprule
    &\multicolumn{3}{c}{Feature dimensions}\\
    \cmidrule{2-4}
    Linear correlation      & 10       & 20      & 40    \\
    \midrule
    Upper Bound             & 10      & 20    &  40   \\
    \midrule
    $\mathbf{f}_{\text{HGR}}(\mathbf{X})$ and $\mathbf{g}_{\text{HGR}}(\mathbf{Y})$         & 1.36    & 2.37     & 3.40        \\
    $\mathbf{f}_{\text{SHGR}}(\mathbf{X})$ and $\mathbf{g}_{\text{SHGR}}(\mathbf{Y})$        & 1.36    & 2.37     & 3.40        \\
    $\mathbf{f}_{\text{SHGR}}(\mathbf{X})$ and $\mathbf{f}_{\text{HGR}}(\mathbf{X})$
                & 9.99     & 20.00      & 39.99    \\
    $\mathbf{g}_{\text{SHGR}}(\mathbf{Y})$ and $\mathbf{g}_{\text{HGR}}(\mathbf{Y})$
                & 10.00    & 20.00      & 39.99   \\
    \bottomrule
  \end{tabular}
  \end{center}
\end{table}

\subsection{Comparing Soft-HGR with HGR} \label{4.1}
The formulation of the Soft-HGR and the original HGR maximal correlation are equivalent except for the way they control whitening. In this section, we compare two methods in terms of linearity, efficiency and stability. 

\subsubsection{Linearity Check} Based on the theory, the HGR and the Soft-HGR transformations should span the same feature space. To verify this, we randomly generated 100K data samples $(x_i, y_i)$ from a randomly chosen joint distribution $P_{XY}$, where $X,Y \in \{1, ..., 50\}$ are both discrete random variables. The HGR feature $\left\{\mathbf{f}_{\text{HGR}}, \mathbf{g}_{\text{HGR}}\right\}$ is obtained by directly solving the SVD for $\tilde{\mathbf{B}}$, which is calculated from empirical joint distribution $\tilde{P}_{XY}$. In order to retrieve the Soft-HGR features $\left\{\mathbf{f}_{\text{SHGR}}, \mathbf{g}_{\text{SHGR}}\right\}$, we first turn the data into one-hot form $\mathbf{X},\mathbf{Y} \in \mathbb{R}^{100K \times 50}$ , then feed them into a two-branch one-layer neural network optimized by Soft-HGR objective. Note that when data are one-hot encoded, all possible functions can be captured by linear operations. Finally, we apply all learned functions to data and run linear CCA between every two feature transformations. Recall that the HGR and Soft-HGR features are linearly transformable from the one to the other. Therefore, the linear correlation between $\{\mathbf{f}_{\text{SHGR}}(\mathbf{X})$, $\mathbf{f}_{\text{HGR}}(\mathbf{X})\}$ and between $\{\mathbf{g}_{\text{SHGR}}(\mathbf{Y})$, $\mathbf{g}_{\text{HGR}}(\mathbf{Y})\}$, in the ideal case, should reach the upper bound. 

Table \ref{table_1} summarizes the simulation result. The HGR and the Soft-HGR extract exactly the same linear correlation between $X$ and $Y$ on different choice of $k$. Besides, the correlation between corresponding features from two models is almost identical to the upper bound, which provides an empirical evidence for our theory.

\begin{table}[t]
  \caption{Phonetic prediction accuracy obtained by different methods on certain percentages of the labeled data in XRMB.}
  \begin{center}
  \begin{tabular}{llll}
    \toprule
    &\multicolumn{3}{c}{Percentages of labels}\\
    \cmidrule{2-4}
    Method                  & 10\%              & 50\%              &  100\%    \\
    \midrule
    Baseline DNN            & 72.2\%            & 81.2\%            &  86.4\%   \\
    \midrule
    PCA + DNN               & 71.5\%            & 80.5\%            &  85.2\%   \\
    CCA + DNN               & 70.7\%            & 79.9\%            &  84.4\%   \\
    Deep CCA + DNN          & 73.2\%            & 80.1\%            &  84.0\%   \\
    Soft-HGR + DNN          & 73.0\%            & 79.9\%            &  83.7\%   \\
    Soft CCA + DNN          & 69.4\%            & 76.0\%            &      78.8\%   \\
    CorrNet                 & 71.2\%            & 79.7\%            &  83.2\%   \\
    \midrule
    Semi Soft-HGR           & \textbf{76.3\%}   & \textbf{85.0\%}   &  \textbf{88.0\%}   \\
    Semi Soft CCA           & 73.6\%   & 82.8\%   & 85.5\%   \\
    \bottomrule
  \end{tabular}
  \end{center}
  \label{table_2}
\end{table}

\subsubsection{Efficiency and Stability} 
In this subsection, we focus on the efficiency and stability provided by two methods in optimization. In particular, we compare the execution time and maximally reachable feature dimension by applying both models to the MNIST handwritten image dataset \cite{lecun1998gradient}, which consists of 60K/10K gray-scale digit images of size $28 \times 28$ as training/testing sets. We follow the experiment setting in \cite{andrew2013deep}, and treat left and right halves of digit images as two modalities $X$ and $Y$. In order to highlight efficiency difference brought by the objectives, we restrict the all the feature transformation to take the linear form. Therefore, the HGR maximal correlation degrades to linear CCA. Both optimizations are executed on a Nvidia Tesla K80 GPU with mini-batch SGD of 5K batchsize. 

Figure \ref{fig:figure_2} compares the execution time on one training epoch with different feature dimensions $k$. As we expected, Soft-HGR is faster than CCA methods by orders of magnitude. In addition, the execution time of CCA method grows quickly with the feature dimensions. This is undesirable in real-world settings where $k$ could be very large. It is also worth noting that CCA experiences numerical issue when feature dimension exceeds 350. The instability arises in that the empirical covariance matrices over some mini-batches become ill-posed, or even non-invertible.

\begin{figure}[t]
  \centering
    \begin{tikzpicture}[scale=0.8]
    \begin{axis}[
        xlabel={Feature dimensions $k$},
        ylabel={Execution time / Seconds},
        xmin=40, xmax=510,
        ymin=0, ymax=23,
        xtick={50, 100, 200, 300, 400, 500},
        ytick={0,5, 10,15, 20},
        legend pos=north west,
        ymajorgrids=true,
        grid style=dashed,
    ]

    \addplot[
        color=blue,
        error bars/.cd,
        y dir=both,
        y explicit
        ]
        coordinates {
        (50,0.27) +- (0.01, 0.01)
        (100,0.29)+- (0.04, 0.04)
        (150,0.30)+- (0.02, 0.02)
        (200,0.34) +- (0.06, 0.06)
        (300,0.36)+- (0.02, 0.02)
        (400,0.40)+- (0.03, 0.03)
        (500,0.45)+- (0.01, 0.01)
        };
        \addlegendentry{Linear Soft-HGR}

    \addplot[
        color=red,
        error bars/.cd,
        y dir=both,
        y explicit
        ]
        coordinates {
        (50,3.17) +- (0.07, 0.07)
        (100,6.03)+- (0.07, 0.07)
        (150,8.82)+- (0.08, 0.08)
        (200,11.23) +- (0.12, 0.12)
        (300,17.08)+- (0.12, 0.12)
        (350, 20.05) +- (0.43,0.43)
        };
        \addlegendentry{CCA}
        \draw [dashed, gray!30!black] (axis cs:350,0) -- (axis cs:350,23) node[pos=0.55, xshift = 18]{$k = 350$};
    \end{axis}
    \end{tikzpicture}
    \caption{Execution time of SGD on CCA and linear Soft-HGR for one training epoch on MNIST data. When $k$ is larger than 350, CCA experiences numerical issues.}
    \label{fig:figure_2}
\end{figure}
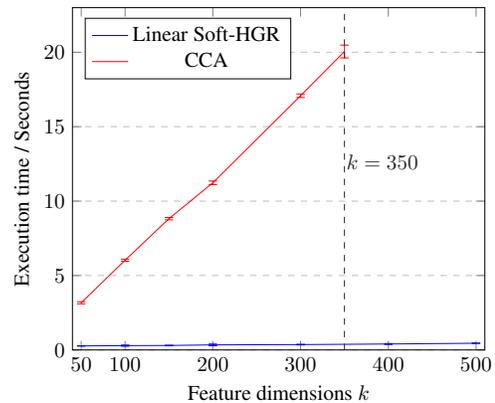

\begin{figure}
  \centering
    \begin{tikzpicture}[scale=0.8]
    \begin{axis}[
        xlabel={Soft-HGR strength $\lambda$},
        ylabel={AUC},
        xmin=0, xmax=0.5,
        ymin=0.655, ymax=0.705,
        xtick={0, 0.1, 0.2, 0.3, 0.4, 0.5},
        ytick={0.66, 0.67, 0.68, 0.69, 0.70},
        legend pos=north west,
        ymajorgrids=true,
        grid style=dashed,
    ]

    \addplot[
        color=red,
        mark=triangle,
        error bars/.cd,
        y dir=both,
        y explicit
        ]
        coordinates {
        (0.0,0.676)
        (0.05,0.691)
        (0.1,0.695)
        (0.15, 0.696)
        (0.2,0.697)
        (0.25, 0.694)
        (0.3, 0.689)
        (0.4, 0.686)
        (0.5, 0.678)
        };
    \end{axis}
    \end{tikzpicture}
    \caption{The effect of hyper-parameter $\lambda$ on AUC}
    \label{fig:figure_3}
\end{figure}
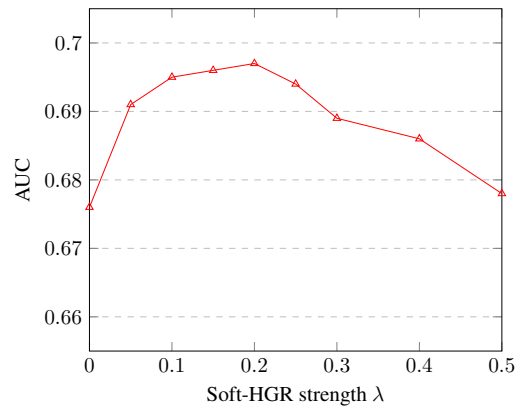

\subsection{Soft-HGR for Semi-supervised Learning}
In this section we demonstrate how Soft-HGR are applied to improve the performance of discriminative tasks. we evaluate our model on the University of Wisconsin X-ray Microbeam Database (XRMB) \cite{westbury1994xrmb} for phonetic classification. XRMB is a bi-modal Dataset consisting of articulatory and acoustic data. Followed the same preprocessing and reconstruction procedures as described in \cite{arora2013multi,wang2015unsupervised}, we obtain the total number of 160K entries of acoustic and vectors articulatory vectors $X \in \mathbb{R}^{273}$ and $Y \in \mathbb{R}^{112}$, corresponding to 41 classes of labels $Z$.

\subsubsection{Experiment Settings} 
While both modalities $X$ and $Y$ are available in the training phase, $Y$ is not provided at the test time. Namely, the model is evaluated by the classification accuracy with only $X$ observed. We expect using $Y$ during training to improve the classification performance, even if they are absent in the test phase. In addition, we partially mask out some portions of labels $Z$ associated with the training data. These two restrictions are consistent with the real world multimodal settings where facial movement data is usually not obtainable, and labels are limited.

\begin{table*}[t]
\caption{AUC obtained by different methods using part of the labels or part of the modalities} \label{table_3}
\begin{center}
\begin{tabular}{m{2.3cm}<{\raggedright} m{1.8cm}<{\raggedright} m{1.3cm}<{\raggedright} m{1.3cm}<{\raggedright} m{1.3cm}<{\raggedright} m{1.3cm}<{\raggedright} m{1.3cm}<{\raggedright}}
    \toprule
    &&\multicolumn{3}{c}{Missing labels} &\multicolumn{2}{c}{Missing modalities}\\
    \cmidrule(l{2pt}r{2pt}){3-5} \cmidrule(l{2pt}r{2pt}){6-7}
    Method                  & No Missing    & 20\%    & 50\%      &  90\%    & 20\%      &  50\% \\
    \midrule
    LR                      & 0.6625    & 0.6623 & 0.6618    &  0.6534        &0.6588     &0.6286\\
    FM                      & 0.6780   & 0.6728  & 0.6723    &  0.6543        &0.6696     &0.6449\\
    Deep FM                 & 0.6803   & 0.6765  & 0.6756    &  0.6613        & 0.6714        & 0.6450\\
    Neural FM               & 0.6760   & 0.6768  & 0.6746    &  0.6570        &0.6661     &0.6574\\
    Semi Soft-HGR           & \textbf{0.6972}   & \textbf{0.6935} & \textbf{0.6906}    &  \textbf{0.6728}   &\textbf{0.6823} & \textbf{0.6682}\\
    \bottomrule
  \end{tabular}
\end{center}
\end{table*}

\subsubsection{Comparing Models} (1) \textbf{Supervised DNN}, which has four hidden layers [1K, 1K, 1K, 1K]. It only takes raw feature $\mathbf{X}$ as inputs and makes predictions on $Z$; Supervised DNN can only deal with labeled data. (2) \textbf{DNN on CCA features}: the DNN structure is the same as in (1) but it accepts transformed $\mathbf{f}(\mathbf{X})$ as input. $\mathbf{f}(\mathbf{X})$ is obtained from PCA, CCA \cite{hotelling1936relations}, Deep CCA \cite{andrew2013deep}, Soft CCA \cite{Chang_2018_CVPR}, Correlational Neural Network (CorrNet) \cite{chandar2016correlational} and our model. Except for PCA which extracts feature only from $\mathbf{X}$, all other methods are trying to find the most correlated $\mathbf{f}(\mathbf{X})$ to $\mathbf{g}(\mathbf{Y})$. For Deep CCA, Soft CCA and our model, the selected $\mathbf{f}$ is a DNN with two hidden layers: [1K, 1K] and $\mathbf{g}$ is linear. The output feature dimensions $k$ is chosen to be 80 for all methods, as higher value leads to unstable gradients in Deep CCA. (3) \textbf{Semi-supervised Model}: we construct semi-supervised Soft-HGR as described in subsection \nameref{sec:semi}, except that the top layer softmax function only takes $\mathbf{f}(X)$ as input. In particular, we use DNN with four hidden layers [1K, 1K, 1K, 1K] for $\mathbf{f}$ and linear function for $\mathbf{g}$. To see the equivalence, when $\lambda = 0$, the network becomes the supervised DNN. For a fair comparison, we also adapt semi-supervised Soft CCA in the same manner. However, we found Deep CCA fails the adaptation because the training is very unstable which prevents us to get a reliable result.
In all DNNs, batch normalization \cite{ioffe2015batch} is applied before ReLU activation function to ensure better convergence. The hyper-parameters for each model are determined by their best average performance on validation set on 5-fold cross validation. Table \ref{table_2} reports the average phonetic prediction accuracy.

\subsubsection{Observations} 
(1) Semi-supervised Soft-HGR achieves the highest accuracy among all models, and the difference becomes more apparent when labels are insufficient. (2) The discriminative performance of Deep CCA and Soft-HGR are similar as they learns equivalent features. (3) $\mathbf{f}(\mathbf{X})$ trained by various unsupervised models is not necessarily more discriminative than raw feature $X_1$. In fact, they only improve classification when labels are extremely limited. In other cases, their performances are inferior to the end-to-end DNN because valuable information may be lost as $\mathbf{f}$ projects the data into lower dimensions. 

\subsection{Soft-HGR for More or Missing Modalities}
In this section we apply our method to recommender system. In such problems, users $X_u$, items $X_i$, and context $X_c$ are three natural modalities. Extensive success achieved by collaborative filtering techniques \cite{breese1998empirical} demonstrates that the correlations between these modalities are useful to infer user behaviors. 

Specifically, we experiment with KKBox's Music Recommendation Dataset \cite{chen2018wsdm}. The goal is to predict the chances of a user listening to a song repetitively after the first listening event within a month. The binary labels $Y=1$ represents the user listens to the song again, and $Y=0$ means the opposite. The user features $X_u$ and item (song) features $X_i$ are explicitly given, and we treat \texttt{source\_system\_tab}, \texttt{source\_screen\_name} and \texttt{source\_type} as context features $X_c$. The categorical features are one-/multi-hot encoded, and continuous ones are normalized. The features corresponding to one modality are concatenated into a single vector, resulting in $X_u$, $X_i$, and $X_c$ as 34656, 623691, and 45 dimensional feature vectors, respectively. The test labels are not disclosed, therefore we use the last $20\%$ of 7M training data as test set\footnote{The split is suggested by the 1st place solution. The last part of the data is used for test set because the data are speculated to be chronologically ordered.}. We test the model under two settings, where labels are insufficient or one modality is missing. In the first setting, we conceal 20\%/50\%/90\% of the labels in training data. In the second scenarios, we randomly mask one of the three modalities as missing in 20\%/50\% of both training and test data, the status of whether data is missing is constructed as a binary flag in the feature vector. The performances are evaluated by the Area under the ROC curve (AUC).

\subsubsection{Comparing Models} We compare our model against to the state-of-art predictive model for sparse data. These include \textbf{Shallow models}: Logistic regression (LR) and Factorization machines (FM) \cite{rendle2010factorization}, and \textbf{Deep models}: Deep FM \cite{guo2017deepfm}, Neural FM \cite{he2017neural}. For models besides LR, the dimension of feature embedding is set to 16. The DNN component for Deep FM, Neural FM and Semi Soft-HGR has consistent structure [100, 100]. \textbf{Semi-supervised Soft-HGR}: The architecture for the unsupervised part is designed mainly according to Figure \ref{fig:figure_1}. However, since fully connected layers is not effective for sparse features, a Bi-Interaction layer, proposed in \cite{he2017neural}, is inserted between the input and DNN structure. The output features from three neural network branches are forwarded to an average pooling layer. The output joint representation is feed to a softmax function for prediction. The hyper-parameter $\lambda$ controls the participation of the Soft-HGR loss. The comparing result is reported in Table \ref{table_3}. In order to highlight the role of Soft-HGR loss, we plot the AUC versus $\lambda$ when Semi Soft-HGR is trained with all labels in Figure \ref{fig:figure_3}.

\subsubsection{Observations} (1) Semi-supervised Soft-HGR achieves significantly better performance than all the baselines. (2) From Figure \ref{fig:figure_3} we can see the performance decreases as it is eliminated from the objective (\ie $\lambda = 0)$. Arguably, the performance gain comes from the introduction of unsupervised Soft-HGR objective.

\section{Conclusion}
In this paper, we propose a multimodal feature extraction framework based on the HGR maximal correlation. Further, we replace the intrinsic whitening constraints with a ``soft'' regularizer which guarantees the efficiency and stability in optimization. Our model is able to cope with more than two modalities, missing modalities, and can be readily generalized to the semi-supervised setting. Extensive experiments show that our proposed model outperforms state-of-the-art multimodal feature selection methods in different scenarios.

\section{Acknowledgement}
The research of Shao-Lun Huang was funded by the Natural Science Foundation of China 61807021, and Shenzhen Municipal Scientific Program JCYJ20170818094022586.

\fontsize{9.0pt}{10.0pt} \selectfont
\bibliography{shgrbib}

\begin{thebibliography}{}

\bibitem[\protect\citeauthoryear{Akaho}{2006}]{akaho2006kernel}
Akaho, S.
\newblock 2006.
\newblock A kernel method for canonical correlation analysis.
\newblock {\em CoRR} abs/cs/0609071.

\bibitem[\protect\citeauthoryear{Andrew \bgroup et al\mbox.\egroup
  }{2013}]{andrew2013deep}
Andrew, G.; Arora, R.; Bilmes, J.; and Livescu, K.
\newblock 2013.
\newblock Deep canonical correlation analysis.
\newblock In {\em International Conference on Machine Learning (ICML)},
  1247--1255.

\bibitem[\protect\citeauthoryear{Arora and Livescu}{2013}]{arora2013multi}
Arora, R., and Livescu, K.
\newblock 2013.
\newblock Multi-view cca-based acoustic features for phonetic recognition
  across speakers and domains.
\newblock In {\em IEEE International Conference on Acoustics, Speech and Signal
  Processing (ICASSP)},  7135--7139.

\bibitem[\protect\citeauthoryear{Bach and Jordan}{2002}]{bach2002kernel}
Bach, F.~R., and Jordan, M.~I.
\newblock 2002.
\newblock Kernel independent component analysis.
\newblock {\em Journal of Machine Learning Research} 3(Jul):1--48.

\bibitem[\protect\citeauthoryear{Baltru\v{s}aitis, Ahuja, and
  Morency}{2018}]{baltrusaitis2018multimodal}
Baltru\v{s}aitis, T.; Ahuja, C.; and Morency, L.-P.
\newblock 2018.
\newblock Multimodal machine learning: A survey and taxonomy.
\newblock {\em IEEE Transactions on Pattern Analysis and Machine Intelligence}
  1--1.

\bibitem[\protect\citeauthoryear{Breese, Heckerman, and
  Kadie}{1998}]{breese1998empirical}
Breese, J.~S.; Heckerman, D.; and Kadie, C.
\newblock 1998.
\newblock Empirical analysis of predictive algorithms for collaborative
  filtering.
\newblock In {\em Conference on Uncertainty in Artificial Intelligence (UAI)},
  43--52.

\bibitem[\protect\citeauthoryear{Chandar \bgroup et al\mbox.\egroup
  }{2016}]{chandar2016correlational}
Chandar, S.; Khapra, M.~M.; Larochelle, H.; and Ravindran, B.
\newblock 2016.
\newblock Correlational neural networks.
\newblock {\em Neural Computation} 28(2):257--285.

\bibitem[\protect\citeauthoryear{Chang, Xiang, and
  Hospedales}{2018}]{Chang_2018_CVPR}
Chang, X.; Xiang, T.; and Hospedales, T.~M.
\newblock 2018.
\newblock Scalable and effective deep cca via soft decorrelation.
\newblock In {\em The IEEE Conference on Computer Vision and Pattern
  Recognition (CVPR)}.

\bibitem[\protect\citeauthoryear{Chen \bgroup et al\mbox.\egroup
  }{2018}]{chen2018wsdm}
Chen, Y.; Xie, X.; Lin, S.-D.; and Chiu, A.
\newblock 2018.
\newblock Wsdm cup 2018: Music recommendation and churn prediction.
\newblock In {\em ACM International Conference on Web Search and Data Mining
  (WSDM)},  8--9.

\bibitem[\protect\citeauthoryear{Eckart and
  Young}{1936}]{eckart1936approximation}
Eckart, C., and Young, G.
\newblock 1936.
\newblock The approximation of one matrix by another of lower rank.
\newblock {\em Psychometrika} 1(3):211--218.

\bibitem[\protect\citeauthoryear{Feizi \bgroup et al\mbox.\egroup
  }{2017}]{feizi2017network}
Feizi, S.; Makhdoumi, A.; Duffy, K.; Kellis, M.; and Medard, M.
\newblock 2017.
\newblock Network maximal correlation.
\newblock {\em IEEE Transactions on Network Science and Engineering}
  4(4):229--247.

\bibitem[\protect\citeauthoryear{Frome \bgroup et al\mbox.\egroup
  }{2013}]{frome2013devise}
Frome, A.; Corrado, G.~S.; Shlens, J.; Bengio, S.; Dean, J.; Mikolov, T.;
  et~al.
\newblock 2013.
\newblock Devise: A deep visual-semantic embedding model.
\newblock In {\em Advances in Neural Information Processing Systems (NIPS)},
  2121--2129.

\bibitem[\protect\citeauthoryear{Gebelein}{1941}]{gebelein1941statistische}
Gebelein, H.
\newblock 1941.
\newblock Das statistische problem der korrelation als variations-und
  eigenwertproblem und sein zusammenhang mit der ausgleichsrechnung.
\newblock {\em ZAMM-Journal of Applied Mathematics and Mechanics/Zeitschrift
  f{\"u}r Angewandte Mathematik und Mechanik} 21(6):364--379.

\bibitem[\protect\citeauthoryear{Guo \bgroup et al\mbox.\egroup
  }{2017}]{guo2017deepfm}
Guo, H.; Tang, R.; Ye, Y.; Li, Z.; and He, X.
\newblock 2017.
\newblock {DeepFM}: A factorization-machine based neural network for ctr
  prediction.
\newblock In {\em International Joint Conference on Artificial Intelligence
  (IJCAI)},  1725--1731.

\bibitem[\protect\citeauthoryear{He and Chua}{2017}]{he2017neural}
He, X., and Chua, T.-S.
\newblock 2017.
\newblock Neural factorization machines for sparse predictive analytics.
\newblock In {\em International ACM SIGIR Conference on Research and
  Development in Information Retrieval (SIGIR)},  355--364.

\bibitem[\protect\citeauthoryear{Hirschfeld}{1935}]{hirschfeld1935connection}
Hirschfeld, H.~O.
\newblock 1935.
\newblock A connection between correlation and contingency.
\newblock {\em Mathematical Proceedings of the Cambridge Philosophical Society}
  31(4):520–524.

\bibitem[\protect\citeauthoryear{Hotelling}{1936}]{hotelling1936relations}
Hotelling, H.
\newblock 1936.
\newblock Relations between two sets of variates.
\newblock {\em Biometrika} 28(3/4):321--377.

\bibitem[\protect\citeauthoryear{Huang \bgroup et al\mbox.\egroup
  }{2017}]{huang2017information1}
Huang, S.-L.; Makur, A.; Zheng, L.; and Wornell, G.~W.
\newblock 2017.
\newblock An information-theoretic approach to universal feature selection in
  high-dimensional inference.
\newblock In {\em International Symposium on Information Theory (ISIT)},
  1336--1340.

\bibitem[\protect\citeauthoryear{Ioffe and Szegedy}{2015}]{ioffe2015batch}
Ioffe, S., and Szegedy, C.
\newblock 2015.
\newblock Batch normalization: Accelerating deep network training by reducing
  internal covariate shift.
\newblock In {\em International Conference on Machine Learning (ICML)},
  448--456.

\bibitem[\protect\citeauthoryear{LeCun \bgroup et al\mbox.\egroup
  }{1998}]{lecun1998gradient}
LeCun, Y.; Bottou, L.; Bengio, Y.; and Haffner, P.
\newblock 1998.
\newblock Gradient-based learning applied to document recognition.
\newblock {\em Proceedings of the IEEE} 86(11):2278--2324.

\bibitem[\protect\citeauthoryear{Pearson}{1895}]{pearson1895note}
Pearson, K.
\newblock 1895.
\newblock Note on regression and inheritance in the case of two parents.
\newblock {\em Proceedings of the Royal Society of London} 58:240--242.

\bibitem[\protect\citeauthoryear{Rendle}{2010}]{rendle2010factorization}
Rendle, S.
\newblock 2010.
\newblock Factorization machines.
\newblock In {\em IEEE International Conference on Data Mining (ICDM)},
  995--1000.

\bibitem[\protect\citeauthoryear{R{\'e}nyi}{1959}]{renyi1959measures}
R{\'e}nyi, A.
\newblock 1959.
\newblock On measures of dependence.
\newblock {\em Acta Mathematica Hungarica} 10(3-4):441--451.

\bibitem[\protect\citeauthoryear{Sohn, Shang, and Lee}{2014}]{sohn2014improved}
Sohn, K.; Shang, W.; and Lee, H.
\newblock 2014.
\newblock Improved multimodal deep learning with variation of information.
\newblock In {\em Advances in Neural Information Processing Systems (NIPS)}.
  Curran Associates, Inc.
\newblock  2141--2149.

\bibitem[\protect\citeauthoryear{Srivastava and
  Salakhutdinov}{2012}]{srivastava2012multimodal}
Srivastava, N., and Salakhutdinov, R.~R.
\newblock 2012.
\newblock Multimodal learning with deep boltzmann machines.
\newblock In {\em Advances in neural information processing systems},
  2222--2230.

\bibitem[\protect\citeauthoryear{Vendrov \bgroup et al\mbox.\egroup
  }{2015}]{vendrov2015order}
Vendrov, I.; Kiros, R.; Fidler, S.; and Urtasun, R.
\newblock 2015.
\newblock Order-embeddings of images and language.
\newblock {\em arXiv preprint arXiv:1511.06361}.

\bibitem[\protect\citeauthoryear{Wang \bgroup et al\mbox.\egroup
  }{2015a}]{wang2015unsupervised}
Wang, W.; Arora, R.; Livescu, K.; and Bilmes, J.~A.
\newblock 2015a.
\newblock Unsupervised learning of acoustic features via deep canonical
  correlation analysis.
\newblock In {\em IEEE International Conference on Acoustics, Speech and Signal
  Processing (ICASSP)},  4590--4594.

\bibitem[\protect\citeauthoryear{Wang \bgroup et al\mbox.\egroup
  }{2015b}]{wang2015stochastic}
Wang, W.; Arora, R.; Livescu, K.; and Srebro, N.
\newblock 2015b.
\newblock Stochastic optimization for deep cca via nonlinear orthogonal
  iterations.
\newblock In {\em Communication, Control, and Computing (Allerton), 2015 53rd
  Annual Allerton Conference on},  688--695.
\newblock IEEE.

\bibitem[\protect\citeauthoryear{Westbury}{1994}]{westbury1994xrmb}
Westbury, J.
\newblock 1994.
\newblock X-ray microbeam speech production database user's handbook.
\newblock {\em Waisman Center on Mental Retardation \& Human Development
  University of Wisconsin Madison}.

\bibitem[\protect\citeauthoryear{Zhao, Hu, and
  Wang}{2015}]{zhao2015heterogeneous}
Zhao, L.; Hu, Q.; and Wang, W.
\newblock 2015.
\newblock Heterogeneous feature selection with multi-modal deep neural networks
  and sparse group lasso.
\newblock {\em IEEE Transactions on Multimedia} 17(11):1936--1948.

\end{thebibliography}
\bibliographystyle{aaai}
\end{document}